\documentclass{article}

\usepackage[preprint]{colm2026_conference}

\usepackage[utf8]{inputenc}
\usepackage[T1]{fontenc}
\usepackage{hyperref}
\usepackage{url}
\usepackage{booktabs}
\usepackage{amsfonts}
\usepackage{amsmath}
\usepackage{amssymb}
\usepackage{amsthm}
\usepackage{nicefrac}
\usepackage{microtype}
\usepackage[table]{xcolor}
\usepackage{xspace}
\usepackage{graphicx}
\usepackage{bbm}
\usepackage{algorithm}
\usepackage{algorithmic}
\usepackage{multirow}
\usepackage{subcaption}
\usepackage{fontawesome5}
\usepackage{wrapfig}
\usepackage{tcolorbox}
\tcbuselibrary{breakable}

\definecolor{placeholdercolor}{HTML}{B45309}
\definecolor{bestrowblue}{HTML}{E8F0FE}
\definecolor{examplegreen}{HTML}{4D835B}
\definecolor{examplegreenbg}{HTML}{F5FAF5}
\definecolor{examplegreentext}{HTML}{355C3E}

\newtcolorbox{promptbox}[1][]{
  breakable,
  colback=white,
  colframe=black!70,
  coltitle=black,
  colbacktitle=black!5,
  fonttitle=\fontfamily{cmss}\selectfont\bfseries,
  fontupper=\fontfamily{cmr}\selectfont\small,
  arc=6pt,
  boxrule=0.7pt,
  left=12pt, right=12pt, top=10pt, bottom=10pt,
  toptitle=6pt, bottomtitle=6pt,
  before upper={\setlength{\parskip}{5pt}\setlength{\parindent}{0pt}},
  title={#1},
}
\newtcolorbox{examplebox}[1][]{
  breakable,
  colback=examplegreenbg,
  colframe=examplegreen,
  coltitle=examplegreentext,
  colbacktitle=examplegreen!15,
  fonttitle=\fontfamily{cmss}\selectfont\bfseries,
  fontupper=\fontfamily{cmr}\selectfont\small,
  arc=6pt,
  boxrule=0.7pt,
  left=12pt, right=12pt, top=10pt, bottom=10pt,
  toptitle=6pt, bottomtitle=6pt,
  before upper={\setlength{\parskip}{5pt}\setlength{\parindent}{0pt}},
  title={#1},
}
\definecolor{agentblue}{HTML}{2563EB}
\definecolor{envgray}{HTML}{6B7280}
\definecolor{errorred}{HTML}{DC2626}
\newtcolorbox{agentbox}[1][]{
  breakable,
  colback=white,
  colframe=agentblue!50,
  coltitle=white,
  colbacktitle=agentblue!80,
  fonttitle=\fontfamily{cmss}\selectfont\small\bfseries,
  fontupper=\fontfamily{cmr}\selectfont\small,
  arc=4pt,
  boxrule=0.6pt,
  left=8pt, right=8pt, top=6pt, bottom=6pt,
  toptitle=4pt, bottomtitle=4pt,
  before upper={\setlength{\parskip}{4pt}\setlength{\parindent}{0pt}},
  title={#1},
}
\newtcolorbox{envbox}[1][]{
  breakable,
  colback=black!2,
  colframe=envgray!40,
  coltitle=black,
  colbacktitle=envgray!10,
  fonttitle=\fontfamily{cmss}\selectfont\small\bfseries,
  fontupper=\fontfamily{cmr}\selectfont\small,
  arc=4pt,
  boxrule=0.6pt,
  left=8pt, right=8pt, top=6pt, bottom=6pt,
  toptitle=4pt, bottomtitle=4pt,
  before upper={\setlength{\parskip}{4pt}\setlength{\parindent}{0pt}},
  title={#1},
}
\newcommand{\ph}[1]{{\color{placeholdercolor}\texttt{\string{#1\string}}}}

\newcommand{\method}{Skill-SD\xspace}
\newcommand{\expect}{\mathbb{E}}
\newcommand{\kl}{\mathrm{KL}}
\DeclareMathOperator{\clip}{clip}

\newcommand{\dUp}[1]{$_{\textcolor{blue}{\uparrow\!#1}}$}
\newcommand{\dDn}[1]{$_{\textcolor{red}{\downarrow\!#1}}$}
\newcommand{\dZr}{$_{0.0}$}

\newtheorem{theorem}{Theorem}
\newtheorem{proposition}[theorem]{Proposition}

\theoremstyle{definition}
\newtheorem{definition}[theorem]{Definition}

\makeatletter
\def\thanks#1{\protected@xdef\@thanks{\@thanks
        \protect\footnotetext{\hspace{-2em}#1}}}
\makeatother

\title{\method: Skill-Conditioned Self-Distillation for Multi-turn LLM Agents}

\author{
  Hao Wang$^{*1,5\dagger}$, Guozhi Wang$^{*5\ddagger}$, Han Xiao$^{*2}$, Yufeng Zhou$^{5}$, Yue Pan$^{5}$, \\
  \textbf{Jichao Wang$^{1}$, Ke Xu$^{3}$, Yafei Wen$^{5}$, Xiaohu Ruan$^{5}$, Xiaoxin Chen$^{5}$, Honggang Qi$^{4\S}$} \\
  $^1$Hangzhou Institute for Advanced Study, University of Chinese Academy of Sciences\\
  $^2$The Chinese University of Hong Kong\\
  $^3$University of Science and Technology of China\\
  $^4$University of Chinese Academy of Sciences\\
  $^5$vivo AI Lab\\
  \texttt{wanghao251@mails.ucas.ac.cn}\quad \texttt{hgqi@ucas.ac.cn}\quad \texttt{11085439@vivo.com}\\
  $^*$Equal contribution \quad $^\S$Corresponding author \quad $^\ddagger$Project lead \quad $^\dagger$Intern at vivo\\[0.15em]
  {\normalfont\small {\color[HTML]{4A90D9}\faGlobe}\ \href{https://skill-sd.github.io/}{skill-sd.github.io}}
}

\begin{document}

\maketitle

\begin{abstract}
Reinforcement learning (RL) has been widely used to train LLM agents for multi-turn interactive tasks, but its sample efficiency is severely limited by sparse rewards and long horizons. On-policy self-distillation (OPSD) alleviates this by providing dense token-level supervision from a privileged teacher that has access to ground-truth answers. However, such fixed privileged information cannot capture the diverse valid strategies in agent tasks, and naively combining OPSD with RL often leads to training collapse. To address these limitations, we introduce \method, a framework that turns the agent's own trajectories into dynamic training-only supervision. Completed trajectories are summarized into compact natural language \emph{skills} that describe successful behaviors, mistakes, and workflows. These skills serve as dynamic privileged information conditioning only the teacher, while the student always acts under the plain task prompt and learns to internalize the guidance through distillation. To stabilize the training, we derive an importance-weighted reverse-KL loss to provide gradient-correct token-level distillation, and dynamically synchronize the teacher with the improving student. Experimental results on agentic benchmarks demonstrate that \method substantially outperforms the standard RL baseline, improving both vanilla GRPO (+14.0\%/+10.9\% on AppWorld/Sokoban) and vanilla OPD (+42.1\%/+40.6\%).
\end{abstract}

\section{Introduction}
\label{sec:intro}

Reinforcement learning (RL) has become the dominant paradigm for post-training LLM agents on multi-turn interactive tasks. These agents typically act over ReAct-style trajectories that interleave reasoning with external tool calls~\citep{yao2023react}, operating software through APIs~\citep{liu2025llm,trivedi2024appworld}, navigating web interfaces~\citep{yao2022webshop}, and solving sequential planning problems. Group Relative Policy Optimization (GRPO)~\citep{shao2024deepseekmath} enables effective policy optimization without a separate value network, and recent work shows that GRPO-based training can produce agents surpassing much larger models~\citep{chen2025loop,feng2025gigpo,zhang2025agentrl}. However, RL inherently suffers from sparse, delayed reward signals and high sample complexity. In long-horizon agentic tasks, the reward is often binary (task completed or not), making it extremely sparse and providing nearly negligible guidance about \emph{which} tokens or actions are actually useful.

On-policy distillation (OPD) addresses this gap by providing dense token-level supervision from a privileged teacher. On-policy self-distillation (OPSD) further removes the need for a seperate teacher model: the same model serves as both teacher (with privileged information) and student (without), enabling efficient single-model training. Self-Distilled Reasoner~\citep{zhao2026sdr} matches GRPO performance with 8--12$\times$ fewer tokens; SDPO~\citep{hubotter2026sdpo} reaches GRPO-level accuracy 6$\times$ faster; KDRL~\citep{xu2025kdrl} further shows that combining distillation with policy gradients outperforms either alone. These successes, however, are largely confined to tasks with unique ground-truth answers such as mathematical proofs and code solutions, where a fixed correct answer naturally serves as the teacher's privileged information.

Extending self-distillation to multi-turn agent tasks raises two challenges that prior work has not addressed. \textbf{(1)~What should the teacher know?} Unlike math problems with a single verifiable answer, agent tasks usually admit diverse valid strategies: an AppWorld task can be solved through different API call sequences, and a Sokoban puzzle can be approached from multiple directions. Conditioning the teacher on any single fixed solution constrains the student's exploration and fails to provide rich guidance. \textbf{(2)~How to keep training stable?} Naively coupling self-distillation with RL for agents often leads to training collapse. When the teacher generates trajectories (off-policy), its distribution drifts from the student's as training progresses, destabilizing the importance-weighted objective.

We present \method, a framework that turns the agent's own trajectory history into a dynamic, training-only teacher signal. \method summarizes completed trajectories into compact natural-language \emph{skills} that capture successful behaviors, mistakes, and high-level workflows. These skills provide reusable strategic guidance rather than a single fixed action sequence, allowing the student to explore its own solutions. 
During training, these skills condition only the teacher; the student generates on-policy trajectories under the plain task prompt. In this way, the student learns to internalize the skills without introducing retrieval dependence at inference time.
In this self-distillation setting, the teacher and the student share the same parameters but are conditioned on different prompts. This prompt discrepancy introduces a distribution mismatch similar to off-policy learning: the student's on-policy tokens are scored under the teacher's skill-augmented distribution.
Recent analyses show that naive $k_3$ differentiation can yield gradient-biased updates even in standard on-policy training~\citep{tang2025pitfalls}. The same issue persists in our cross-prompt setting. To address this, we derive an importance-weighted reverse-KL loss ($\rho \cdot k_3$) that corrects this distribution mismatch, ensuring unbiased token-level gradient updates for the self-distillation objective. Moreover, a frozen teacher inevitably falls behind as the student improves, degrading both the skill-conditioned guidance and the distillation signal. To keep the teacher calibrated, \method dynamically synchronizes the teacher from the latest student checkpoint, so that the privileged signal co-evolves with the student's improving policy. Our contributions are summarized as follows:

\begin{itemize}
    \item \textbf{Skill as dynamic teacher signal.} We propose using trajectory-derived natural-language skills as dynamic privileged information that conditions only the teacher during self-distillation. This preserves the diversity of valid action paths while allows for the skill internalization for the student model (\S\ref{sec:teacher}).

    \item \textbf{Importance-weighted reverse-KL loss.} We derive an importance-weighted reverse-KL loss for the cross-prompt self-distillation, where teacher and student share parameters but differ in prompt conditioning. This loss corrects the per-token gradient bias of the naive $k_3$ estimator under distribution mismatch (\S\ref{sec:sdl}, Appendix~\ref{app:sdl_gradient}).

    \item \textbf{Necessity of dynamic teacher synchronization.} Through systematic ablation, we show that periodically synchronizing the teacher with the improving student is essential for stable training: off-policy teacher-owned rollouts collapse during mid-training, while frozen teachers converge to lower plateaus. (\S\ref{sec:ablation}).
\end{itemize}

We evaluate \method on multiple agentic benchmarks. Using Qwen3-4B-Instruct-2507 as the base model, \method achieves 64.9\% accuracy on AppWorld and 62.5\% on Sokoban, outperforming vanilla GRPO by +14.0\% and +10.9\%, respectively, while also surpassing vanilla OPD by large margins. These results demonstrate the effectiveness of skill-conditioned self-distillation for training LLM agents.

\section{Related Work}
\label{sec:related}

\paragraph{On-policy distillation and self-distillation for LLMs.}
Recent self-distillation methods for LLMs~\citep{furlanello2018born,agarwal2024gkd} differ along two axes. First, some rely on distillation alone~\citep{zhao2026sdr,ye2026opcd}, whereas others integrate distillation with RL---by converting environment feedback into a dense self-distillation target~\citep{hubotter2026sdpo}, unifying KD and RL in a single objective~\citep{xu2025kdrl}, or mitigating their interference through selective imitation~\citep{zhang2026rlad}. Second, the privileged context is typically \emph{fixed or externally supplied}---ground-truth solutions, offline experience, or frontier-model trajectories as in concurrent work pi-Distill~\citep{penaloza2026pidistill}. Concurrent OEL~\citep{ye2026oel} extracts transferable experience from trajectories and consolidates it via context distillation, but uses distillation alone without RL reward signals; OpenClaw-RL~\citep{wang2026openclawrl} combines environment-derived next-state signals with RL for agent training. \method integrates self-distillation with RL in a joint objective, using analytical skills generated from the agent's own rollouts, and formalizes the cross-prompt importance-weighted KL estimator that this setting requires (Appendix~\ref{app:sdl_gradient}).

\paragraph{KL-regularized policy optimization and estimator design.}
A recent insight in the $k_1/k_2/k_3$ KL estimator family~\citep{schulman2017ppo,schulman2020kl} is that KL \emph{value} estimation and KL \emph{gradient} optimization are distinct problems: the widely used $k_3$ estimator provides unbiased value estimates but biased gradients when differentiated directly as a loss~\citep{tang2025pitfalls,liu2025rethinking}. RPG~\citep{zhang2025rpg} further studies the interaction of importance weighting, KL direction, and clipped updates at scale. These analyses consider the standard same-prompt setting; our cross-prompt configuration, where teacher and student share parameters but differ in prompt conditioning, requires the importance weighting derived in Appendix~\ref{app:sdl_gradient}.

\paragraph{Reinforcement learning for multi-turn LLM agents.}
GRPO~\citep{shao2024deepseekmath,deepseekr1} has become the predominant RL algorithm for LLM training due to its memory efficiency. Adapting RL to multi-turn agent settings introduces new challenges: LOOP~\citep{chen2025loop} adapts PPO for long-horizon tasks without a value network, GiGPO~\citep{feng2025gigpo} proposes two-level advantage estimation for multi-turn GRPO, RAGEN~\citep{wang2025ragen} identifies the ``echo trap'' instability, and AgentEvolver~\citep{zhai2025agentevolver} introduces self-attributing mechanisms for fine-grained credit assignment. Agent-R1~\citep{cheng2025agentr1}, AgentGym-RL~\citep{agentgymrl}, and AgentRL~\citep{zhang2025agentrl} scale multi-turn RL across diverse environments with principled environment design~\citep{wang2025practitioner, zhao2025mas, liu2025learnact}. These works improve RL algorithms or infrastructure for agents but do not incorporate self-distillation. \method complements this line by adding an auxiliary distillation loss that transfers skill-conditioned teacher knowledge alongside RL reward signals.

\paragraph{Experience, reflection, and memory in agents.}
Reflexion~\citep{shinn2023reflexion} stores self-reflections in episodic memory, ExpeL~\citep{zhao2023expel} extracts reusable experience, Agent-R~\citep{yuan2025agentr} trains agents to reflect via iterative self-training, ECHO~\citep{hu2025echo} rewrites hindsight trajectories, and EvolveR~\citep{wu2025evolver} synthesizes trajectory outcomes into strategic principles. UI-Genie~\citep{xiao2025ui} iteratively boosts agent performance through a self-improving loop between agent model and reward model, while UI-Mem~\citep{xiao2026ui, liu2026memgui} builds a self-evolving experience memory for online RL.
A common design across these works is to retrieve experience and append it to the agent's prompt to improve trajectory generation. This creates two related tensions: within training, trajectories are generated under $\pi(a \mid h, e)$ but the policy-gradient update targets $\pi(a \mid h, \varnothing)$, inducing an unmodeled behavior--target mismatch in the importance ratio; at inference time, stripping experience causes a performance gap, while retaining it makes the policy dependent on retrieval. \method avoids both issues by conditioning only the teacher on skills during training, while the student generates on-policy trajectories under the plain task prompt. The auxiliary SDL loss distills the teacher's token-level guidance into the student, so the student is trained and evaluated under the same plain prompt, without any retrieval dependency.

\section{Method}
\label{sec:method}

\begin{figure}[t]
\centering
\includegraphics[width=\textwidth]{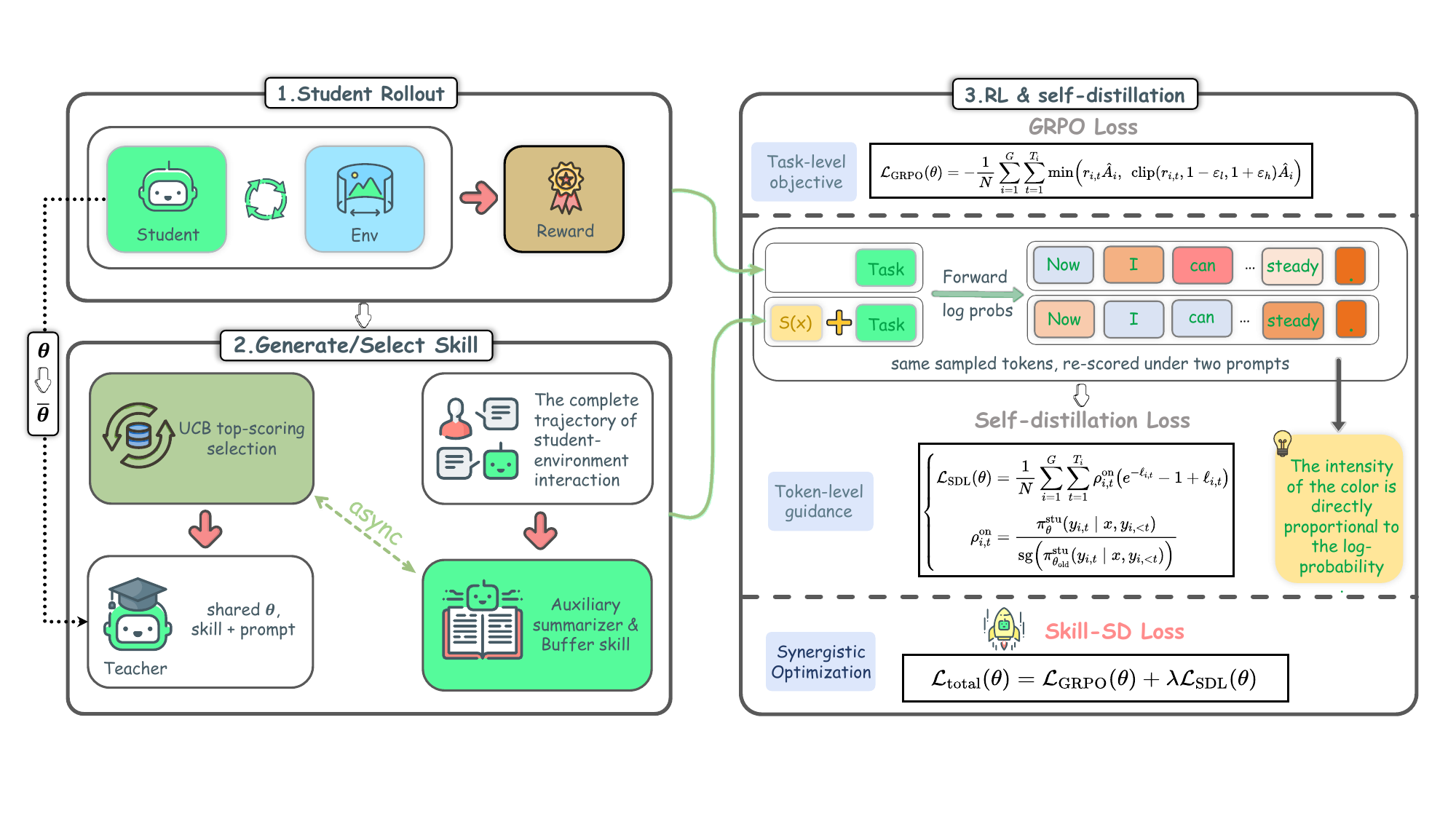}
\caption{\method overview. (1)~The student generates on-policy rollouts and receives task-level rewards. (2)~Completed trajectories are asynchronously summarized into compact skills and stored in a UCB-indexed buffer. (3)~The same token sequence is re-scored under the student prompt and the skill-augmented teacher prompt; GRPO provides task-level credit assignment while the importance-weighted SDL loss transfers token-level teacher knowledge. Token color intensity reflects prediction probability.}
\label{fig:overview}
\end{figure}

\subsection{Problem Setup and GRPO Backbone}
\label{sec:setup}

Figure~\ref{fig:overview} illustrates the overall \method pipeline. Each task instance \(x\) defines a multi-turn interaction with horizon \(H\). A rollout \(\tau = (y_1,\dots,y_T)\) denotes the sequence of action tokens emitted by the agent while interacting with the environment, where \(T \leq H\) counts all action tokens across turns. We use a \emph{completion-rate} reward
\begin{equation}
R(x,\tau) = \frac{1}{K}\sum_{k=1}^{K}\mathbbm{1}\!\left[c_k(\tau)\ \text{is satisfied}\right] \in [0,1],
\label{eq:reward}
\end{equation}
where \(\{c_k\}_{k=1}^{K}\) are task-specific verification criteria: state-based unit tests in AppWorld~\citep{trivedi2024appworld} and boxes-on-target counts in Sokoban (see \S\ref{sec:exp_setup} for details). This fine-grained reward preserves partial progress that would be discarded by a binary success signal.

For each task, we sample a group of \(G\) trajectories from the student policy and normalize rewards within the group:
\begin{equation}
\hat{A}_i = \frac{R_i - \frac{1}{G}\sum_{j=1}^{G} R_j}{\mathrm{std}(R_1,\dots,R_G) + \epsilon},
\label{eq:adv}
\end{equation}
where \(R_i = R(x,\tau_i)\). The GRPO objective is:
\begin{equation}
\mathcal{L}_{\text{GRPO}}(\theta)
=
-\frac{1}{N}\sum_{i=1}^{G}\sum_{t=1}^{T_i}
\min\!\left(
r_{i,t}\hat{A}_i,\;
\clip(r_{i,t},1-\varepsilon_l,1+\varepsilon_h)\hat{A}_i
\right),
\label{eq:grpo}
\end{equation}
with student trust-region ratio
\begin{equation}
r_{i,t}
=
\frac{\pi_{\theta}^{\text{stu}}(y_{i,t}\mid x,y_{i,<t})}
{\operatorname{sg}\!\left(\pi_{\theta_{\text{old}}}^{\text{stu}}(y_{i,t}\mid x,y_{i,<t})\right)},
\label{eq:ratio}
\end{equation}
and \(N=\sum_{i=1}^{G}T_i\) is the total number of valid action tokens in the group, implementing \texttt{token-mean} reduction where both long and short trajectories contribute in proportion to their valid-token counts. Following DAPO~\citep{yu2025dapo}, we use asymmetric clipping bounds with \(\varepsilon_h > \varepsilon_l\) to prevent premature entropy collapse.

\subsection{Skill-Conditioned Teacher and Dynamic Self-Evolution}
\label{sec:teacher}

Unlike math or code generation, multi-turn agent tasks rarely have a unique ground-truth solution: many action sequences can solve the same task. This rules out fixed-reference distillation but leaves room for a softer form of privileged information: structured summaries of what worked and what failed in past attempts. Injecting such summaries into the student would leak privileged information into the evaluation-time policy; conditioning only the teacher preserves a clean student interface while still transferring the knowledge gap.

The teacher in \method is therefore \emph{skill-conditioned}: it receives the same task prompt augmented with compact natural-language skills extracted from prior trajectories. Each skill summarizes three aspects of an attempt:
\begin{equation}
e = \mathrm{Analyze}(\tau,x) =
\bigl(
e_{\text{success}},
e_{\text{mistake}},
e_{\text{workflow}}
\bigr),
\label{eq:exp}
\end{equation}
capturing what worked, what failed, and what high-level workflow should be followed next time. We do not store full trajectories as skills. Multi-turn tasks admit many valid action sequences, and a single canonical trace would overconstrain exploration.

Given task \(x\), the student acts on the plain prompt, while the teacher sees the prompt augmented with retrieved skills:
\begin{align}
\pi_{\theta}^{\text{stu}}(\cdot \mid x, y_{<t})
&=
\pi_{\theta}(\cdot \mid x, y_{<t}), \label{eq:student_policy} \\
\pi_{\bar{\theta}}^{\text{tea}}(\cdot \mid x, S(x), y_{<t})
&=
\pi_{\bar{\theta}}(\cdot \mid x \oplus S(x), y_{<t}). \label{eq:teacher_policy}
\end{align}
Here \(S(x)\) denotes the retrieved skills and \(\bar{\theta}\) is the teacher parameter state. In the \textbf{dynamic} setting, \(\bar{\theta}\) is synchronized from the latest student checkpoint at each iteration; in the \textbf{frozen} setting, \(\bar{\theta}\) is fixed throughout training.

Skill retrieval is lightweight. For each task, we select the single highest-scoring skill using a UCB criterion:
\begin{equation}
\mathrm{score}(e) = \bar{r}(e) + c\sqrt{\frac{\ln N_{\text{ucb}}}{n(e)}},
\label{eq:ucb}
\end{equation}
where \(\bar{r}(e)\) is the mean reward of skill \(e\), \(N_{\text{ucb}}\) is the total number of retrievals for the same task, \(n(e)\) is the number of times \(e\) has been selected for that task, and \(c\) controls the exploration--exploitation tradeoff. Skills with \(n(e)=0\) are selected first, ensuring that newly generated skills are always tried before the UCB score is computed. All statistics are maintained per-task. An auxiliary LLM is used only to summarize trajectories into skills; it does not participate in the optimization objective. This is a task-local form of context distillation~\citep{snell2022learning,hsieh2023distilling}: the teacher gets richer context during training, but the student does not receive it at test time.

\subsection{Importance-weighted reverse-KL loss}
\label{sec:sdl}

The skill-conditioned teacher provides a richer action distribution than the student, but distilling from it is not straightforward. Student and teacher condition on different prompts, so their log-probability ratio is no longer a standard on-policy quantity. Naively differentiating the $k_3$ divergence estimator yields biased gradients even in standard on-policy settings~\citep{tang2025pitfalls}; in our cross-prompt regime the same issue persists, so an explicit importance-correction term is needed to restore per-token unbiasedness.

In \method, trajectories are sampled from the old student under the plain prompt, and the skill-conditioned teacher re-scores those same tokens under the augmented prompt. Specifically, we sample
\begin{equation}
\tau_i \sim \pi_{\theta_{\text{old}}}^{\text{stu}}(\cdot \mid x),
\label{eq:student_rollout}
\end{equation}
while recording the old-student log-probabilities during rollout. We then re-forward the sampled sequence under the current student and under the fixed teacher view \(\pi_{\bar{\theta}}^{\text{tea}}\) to obtain \(\log \pi_{\theta}^{\text{stu}}\) and \(\log \pi_{\bar{\theta}}^{\text{tea}}\).

For each sampled token, define the current-student to teacher-reference log-ratio
\begin{equation}
\ell_{i,t}
=
\log \pi_{\theta}^{\text{stu}}(y_{i,t}\mid x,y_{i,<t})
-
\log \pi_{\bar{\theta}}^{\text{tea}}(y_{i,t}\mid x,S(x),y_{i,<t}),
\label{eq:logratio}
\end{equation}
and the on-policy importance weight
\begin{equation}
\rho_{i,t}^{\text{on}}
=
\frac{\pi_{\theta}^{\text{stu}}(y_{i,t}\mid x,y_{i,<t})}
{\operatorname{sg}\!\left(\pi_{\theta_{\text{old}}}^{\text{stu}}(y_{i,t}\mid x,y_{i,<t})\right)}.
\label{eq:rho}
\end{equation}
We then optimize the auxiliary self-distillation loss
\begin{equation}
\mathcal{L}_{\text{SDL}}(\theta)
=
\frac{1}{N}\sum_{i=1}^{G}\sum_{t=1}^{T_i}
\rho_{i,t}^{\text{on}}\left(
e^{-\ell_{i,t}} - 1 + \ell_{i,t}
\right).
\label{eq:sdl}
\end{equation}

The auxiliary SDL term transfers teacher-side skill knowledge into the student under the correct sampling distribution. The weight \(\rho_{i,t}^{\text{on}}\) provides the necessary importance correction, since differentiating bare $k_3$ directly produces biased gradients~\citep{tang2025pitfalls,liu2025rethinking}. The formal gradient identity is given in Appendix~\ref{app:sdl_gradient}.

For theoretical completeness and for the ablation study in \S\ref{sec:ablation}, we also derive the importance weight under teacher-owned (off-policy) rollout, where trajectories are sampled from \(\pi_{\theta_{\text{old}}}^{\text{tea}}(\cdot \mid x \oplus S(x))\). The same construction yields
\begin{equation}
\rho_{i,t}^{\text{off}}
=
\frac{\pi_{\theta}^{\text{stu}}(y_{i,t}\mid x,y_{i,<t})}
{\operatorname{sg}\!\left(\pi_{\theta_{\text{old}}}^{\text{tea}}(y_{i,t}\mid x,S(x),y_{i,<t})\right)},
\label{eq:rho_off}
\end{equation}
preserving the same gradient-correct interpretation under teacher sampling (Appendix~\ref{app:corollaries}).

Equation~\eqref{eq:sdl} is a \emph{sampled-token} reverse-KL objective rather than a full-vocabulary distillation loss, and it uses the same \texttt{token-mean} aggregation as the GRPO term. For long-horizon agent traces, sampled-token distillation avoids the memory overhead of storing full-vocabulary logits at every token position and requires only the re-forward pass that is already needed for computing $r_{i,t}$.

\paragraph{Visualizing distillation dynamics.}
Figure~\ref{fig:sdl_viz} visualizes how the student and teacher token distributions evolve during training on a representative AppWorld task. At the start of training, the student and teacher assign notably different probabilities to the same tokens; by the end, their distributions have largely converged. The teacher's distribution remains relatively stable across epochs, while the student progressively learns to match it through the SDL loss. The SDL loss curve (right) decreases by 59.3\% over training, confirming this convergence.

\begin{figure}[t]
\centering
\includegraphics[width=\textwidth]{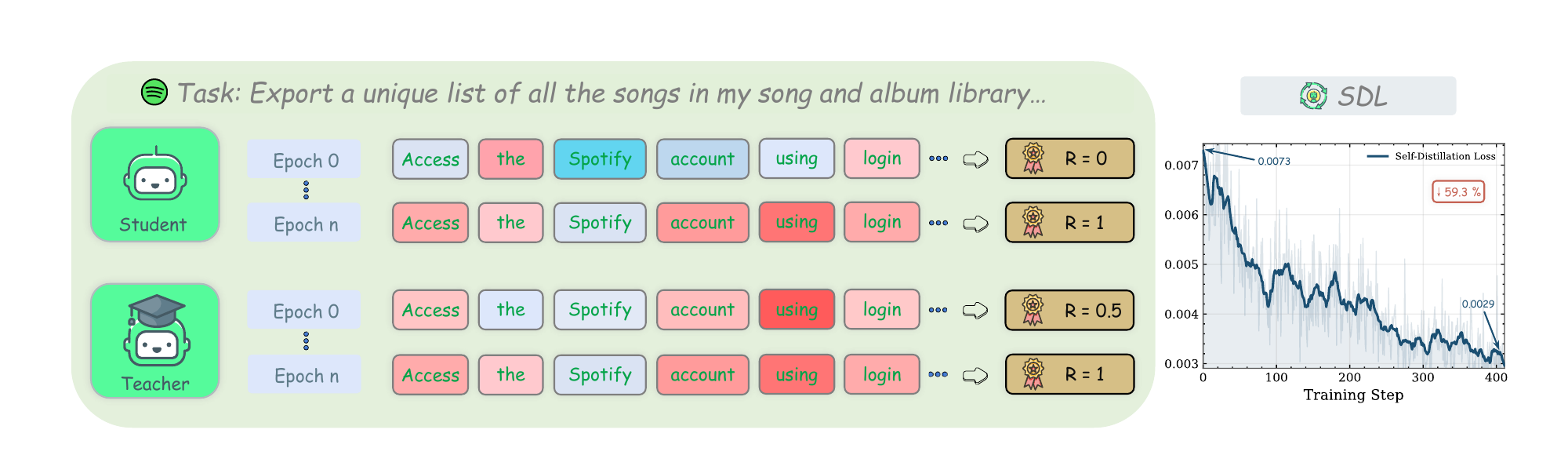}
\caption{Token-level self-distillation dynamics on an AppWorld task. Token color intensity is proportional to prediction probability. At Epoch~0, student and teacher distributions differ substantially; after convergence, they have largely aligned. Right: the SDL loss decreases by 59.3\% over training.}
\label{fig:sdl_viz}
\end{figure}

\subsection{Training Objective and Procedure}
\label{sec:training}

GRPO and SDL are complementary: GRPO uses trajectory-level reward to select good rollouts, while SDL uses the skill-conditioned teacher to redistribute probability mass at the token level within those rollouts. Neither signal alone suffices. RL without per-token guidance is sample-inefficient in long-horizon tasks, and distillation without reward grounding can drift toward the teacher even when the teacher is wrong. Combining them lets the reward filter bad trajectories while the teacher accelerates credit assignment on good ones.

The overall training objective combines GRPO reward maximization with the auxiliary SDL term:
\begin{equation}
\mathcal{L}_{\text{total}}(\theta)
=
\mathcal{L}_{\text{GRPO}}(\theta)
+
\lambda \mathcal{L}_{\text{SDL}}(\theta),
\label{eq:total}
\end{equation}
where \(\lambda\) controls the distillation strength. The two terms act at different resolutions: \(\mathcal{L}_{\text{GRPO}}\) provides task-level credit assignment through group-relative advantages, determining \emph{which} trajectories should be reinforced, while \(\mathcal{L}_{\text{SDL}}\) supplies dense token-level guidance toward the skill-conditioned teacher, shaping \emph{how} probability mass is redistributed along those trajectories. These signals are often complementary, but they can disagree when teacher-preferred actions are not aligned with the reward-improving direction on a given sample. The coefficient \(\lambda\) mediates this trade-off: small \(\lambda\) recovers nearly pure GRPO, while larger \(\lambda\) more strongly constrains the student toward the teacher distribution. Within a given update, Proposition~\ref{prop:sdl_grad} shows that \(\nabla_\theta(\rho \cdot k_3) = \rho \cdot s_\theta \cdot \ell_t\), so the SDL gradient is exactly zero on sampled tokens where teacher and student agree (\(\ell_t = 0\)). In the on-policy regime, where \(\rho_t^{\text{on}}\) stays near \(1\) under small PPO-style updates, the SDL term is correspondingly self-damping as student--teacher disagreement shrinks; Figure~\ref{fig:sdl_viz} confirms this convergence empirically. This property makes SDL an auxiliary shaping signal rather than the primary learning driver, motivating a small \(\lambda\) in practice. Note that although \(\rho_t^{\text{on}}\) and the GRPO clipping ratio \(r_t\) coincide numerically in the main branch, they serve distinct roles---importance correction for SDL versus student-centered trust region for GRPO---as detailed in Appendix~\ref{app:dual_ratio} (Table~\ref{tab:ratio_comparison}). Algorithm~\ref{alg:skill-sd} summarizes a single training iteration.

\begin{algorithm}[t]
\caption{\method training}
\label{alg:skill-sd}
\begin{algorithmic}[1]
\REQUIRE Student policy $\pi_{\theta}$, skill bank $\mathcal{B}$, group size $G$, SDL coefficient $\lambda$, learning rate $\eta$
\FOR{iteration $= 1, 2, \dots$}
    \STATE $\theta_{\text{old}} \leftarrow \theta$;\; sync teacher $\bar{\theta} \leftarrow \theta$ \hfill $\triangleright$ \textit{dynamic}; or keep $\bar{\theta}$ fixed (\textit{frozen})
    \STATE \textbf{--- Rollout ---}
    \FOR{each task $x$ in the batch}
        \STATE $S(x) \leftarrow \mathrm{UCB\text{-}retrieve}\!\left(\mathcal{B}(x)\right)$ \hfill $\triangleright$ Eq.~\ref{eq:ucb}; $\emptyset$ if no prior skills
        \STATE Sample $\{\tau_i\}_{i=1}^{G} \sim \pi_{\theta_{\text{old}}}^{\text{stu}}(\cdot \mid x)$;\; record $\log \pi_{\theta_{\text{old}}}^{\text{stu}}$
        \STATE Compute rewards $R_i = R(x,\tau_i)$ and advantages $\hat{A}_i$ \hfill $\triangleright$ Eq.~\ref{eq:adv}
    \ENDFOR
    \STATE \textbf{--- Re-score ---}
    \FOR{each $(x, \{\tau_i\})$}
        \STATE Compute $\log \pi_{\theta}^{\text{stu}}(y_{i,t} \mid x,\, y_{i,<t})$ \hfill $\triangleright$ current student, plain prompt
        \STATE Compute $\log \pi_{\bar{\theta}}^{\text{tea}}(y_{i,t} \mid x \oplus S(x),\, y_{i,<t})$ \hfill $\triangleright$ teacher, augmented prompt
    \ENDFOR
    \STATE \textbf{--- Policy Update ---}
    \STATE $r_{i,t} = \rho_{i,t}^{\text{on}} \leftarrow \pi_{\theta}^{\text{stu}} \big/ \operatorname{sg}\!\left(\pi_{\theta_{\text{old}}}^{\text{stu}}\right)$ \hfill $\triangleright$ trust-region ratio \& IS weight
    \STATE $\ell_{i,t} \leftarrow \log \pi_{\theta}^{\text{stu}} - \log \pi_{\bar{\theta}}^{\text{tea}}$ \hfill $\triangleright$ student--teacher log-ratio (Eq.~\ref{eq:logratio})
    \STATE $\mathcal{L} \leftarrow \mathcal{L}_{\text{GRPO}}(\theta;\, r,\, \hat{A})\ +\ \lambda\,\mathcal{L}_{\text{SDL}}(\theta;\, \rho^{\text{on}},\, \ell)$ \hfill $\triangleright$ Eqs.~\ref{eq:grpo},\,\ref{eq:sdl}
    \STATE $\theta \leftarrow \theta - \eta\,\nabla_{\theta}\mathcal{L}$
    \STATE \textit{Async:} summarize new trajectories $\to$ skills; update $\mathcal{B}$ \hfill $\triangleright$ consumed next iter.
\ENDFOR
\end{algorithmic}
\end{algorithm}

\section{Experiments}
\label{sec:experiments}

\subsection{Setup}
\label{sec:exp_setup}

\paragraph{Benchmarks.}
We evaluate \method on two multi-turn agentic environments that test complementary capabilities. AppWorld~\citep{trivedi2024appworld} emphasizes real-world API coordination, multi-app state management, and adaptive replanning from environment feedback; errors are generally recoverable through subsequent API calls. Sokoban, by contrast, demands spatial reasoning and long-horizon planning under high irreversibility: a single misstep (e.g., pushing a box into a corner) can render the puzzle unsolvable. Together, the two benchmarks span the spectrum from feedback-driven tool use to deliberative forward planning, and give a more complete picture of agent capabilities.

\textbf{AppWorld}~\citep{trivedi2024appworld} is a comprehensive multi-app API benchmark with 9 simulated consumer applications and 457 APIs. Tasks require multi-step reasoning, cross-application coordination, and iterative API interactions. Correctness is verified by state-based unit tests that check whether the correct state changes were made without undesired side effects. The official training set contains 105 tasks, of which 90 are publicly available~\citep{trivedi2024appworld,chen2025loop}. We train on all 90 available tasks and evaluate on the 57-task development set, with a maximum of $H = 40$ interaction turns.

\textbf{Sokoban} is a classic puzzle-planning benchmark requiring strategic reasoning and long-horizon planning. We use $6 \times 6$ rooms with 2 boxes and a maximum of $H = 40$ steps. Levels are procedurally generated using the reverse-play method of \texttt{gym-sokoban}~\citep{schrader2018gymsokoban}, where the \emph{search depth} parameter controls the depth of the reverse DFS and thus the minimum solution complexity. Since task diversity is limited, we adopt a difficulty curriculum~\citep{bengio2009curriculum} to maximize signal extraction from each level: the training set consists of 96 levels where the easiest 30\% use search depth 15, the middle 50\% use depth 20, and the hardest 20\% use depth 25. The test set consists of 64 levels with uniform search depth 25.

\paragraph{Models and metrics.}
All experiments use Qwen3-4B-Instruct-2507~\citep{qwen3} as the base model (full hyperparameters in Table~\ref{tab:hyperparams}). We report \textbf{accuracy} (pass@1) and \textbf{completion rate} on the AppWorld 57-task development set and the Sokoban 64-level test set. The completion rate is the same dense reward used for training (Eq.~\ref{eq:reward}).

\paragraph{Baselines.}
To isolate the contribution of each component, we compare against:
\begin{itemize}
    \item \textbf{Vanilla GRPO}: standard GRPO with completion-rate reward, without skill prompting or self-distillation.
    \item \textbf{Skill-Augmented GRPO}: GRPO with skills prepended to the rollout prompt, but without the SDL loss ($\lambda = 0$). This isolates the effect of skill prompting from gradient-based distillation.
    \item \textbf{Vanilla OPD}: on-policy distillation without RL reward. The student generates rollouts under the plain prompt; the reverse-KL divergence between the skill-conditioned frozen teacher and the student at each token serves as the per-token advantage in a policy gradient objective. No environment reward is used ($\mathcal{L}_{\text{GRPO}} = 0$), so learning is driven entirely by the distillation signal.
    \item \textbf{Base model}: the Qwen3-4B-Instruct-2507 checkpoint without agent finetuning.
\end{itemize}

\subsection{Main Results}
\label{sec:main_results}

\begin{table}[t]
\centering
\caption{Main results on AppWorld and Sokoban (Qwen3-4B-Instruct-2507). \method uses on-policy student rollout with dynamic teacher synchronization. Subscripts denote absolute change from the base model.}
\label{tab:appworld_main}
\small
\setlength{\tabcolsep}{5pt}
\begin{tabular}{lcccccc}
\toprule
& \multicolumn{2}{c}{\textbf{AppWorld}} & \multicolumn{2}{c}{\textbf{Sokoban}} & \multicolumn{2}{c}{\textbf{Avg.}} \\
\cmidrule(lr){2-3} \cmidrule(lr){4-5} \cmidrule(lr){6-7}
\textbf{Method} & \textbf{Acc.} & \textbf{Comp.} & \textbf{Acc.} & \textbf{Comp.} & \textbf{Acc.} & \textbf{Comp.} \\
\midrule
Base Model & 8.8\dZr & 39.1\dZr & 12.5\dZr & 32.0\dZr & 10.6\dZr & 35.6\dZr \\
Vanilla OPD & 22.8\dUp{14.0} & 59.7\dUp{20.6} & 21.9\dUp{9.4} & 37.5\dUp{5.5} & 22.4\dUp{11.7} & 48.6\dUp{13.0} \\
Vanilla GRPO & 50.9\dUp{42.1} & 76.3\dUp{37.2} & 51.6\dUp{39.1} & 68.8\dUp{36.8} & 51.2\dUp{40.6} & 72.5\dUp{36.9} \\
Skill-Augmented GRPO & 42.1\dUp{33.3} & 76.1\dUp{37.0} & 20.3\dUp{7.8} & 37.5\dUp{5.5} & 31.2\dUp{20.6} & 56.8\dUp{21.2} \\
\midrule
\rowcolor{bestrowblue}
\textbf{\method} & \textbf{64.9}\dUp{56.1} & \textbf{84.9}\dUp{45.8} & \textbf{62.5}\dUp{50.0} & \textbf{71.1}\dUp{39.1} & \textbf{63.7}\dUp{53.1} & \textbf{78.0}\dUp{42.4} \\
\bottomrule
\end{tabular}
\end{table}

Table~\ref{tab:appworld_main} compares \method against all baselines on both benchmarks. On AppWorld, \method achieves 64.9\% accuracy and 84.9\% completion rate, outperforming Skill-Augmented GRPO by +22.8\% accuracy and Vanilla GRPO by +14.0\%. On Sokoban, \method reaches 62.5\% accuracy and 71.1\% completion rate, surpassing Vanilla GRPO by +10.9\%. Skill-Augmented GRPO underperforms Vanilla GRPO on both benchmarks (42.1\% vs.\ 50.9\% on AppWorld; 20.3\% vs.\ 51.6\% on Sokoban), demonstrating that directly injecting skills into the student prompt is counterproductive. Vanilla OPD remains the weakest reward-free method on both benchmarks, confirming that self-distillation alone cannot replace reward-driven optimization.

\paragraph{Training dynamics.}
Figure~\ref{fig:training_baselines} shows the training dynamics of all baselines on both benchmarks. On AppWorld, Skill-Augmented GRPO achieves the highest \emph{training} accuracy but overfits severely: its validation accuracy (42.1\%, Table~\ref{tab:appworld_main}) falls well below \method (64.9\%). \method improves steadily on both training and validation throughout. On Sokoban, \method starts slower than Vanilla GRPO but overtakes it after step~60 and continues to improve steadily. In both environments, Vanilla OPD remains at low performance throughout training.

\begin{figure}[t]
\centering
\begin{subfigure}[t]{0.49\textwidth}
    \centering
    \includegraphics[width=\textwidth]{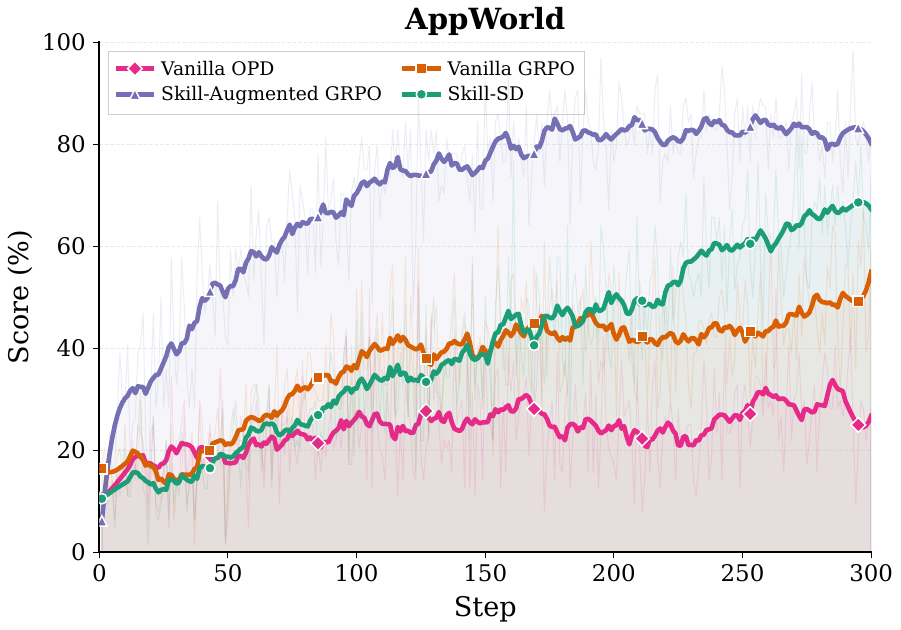}
\end{subfigure}
\hfill
\begin{subfigure}[t]{0.49\textwidth}
    \centering
    \includegraphics[width=\textwidth]{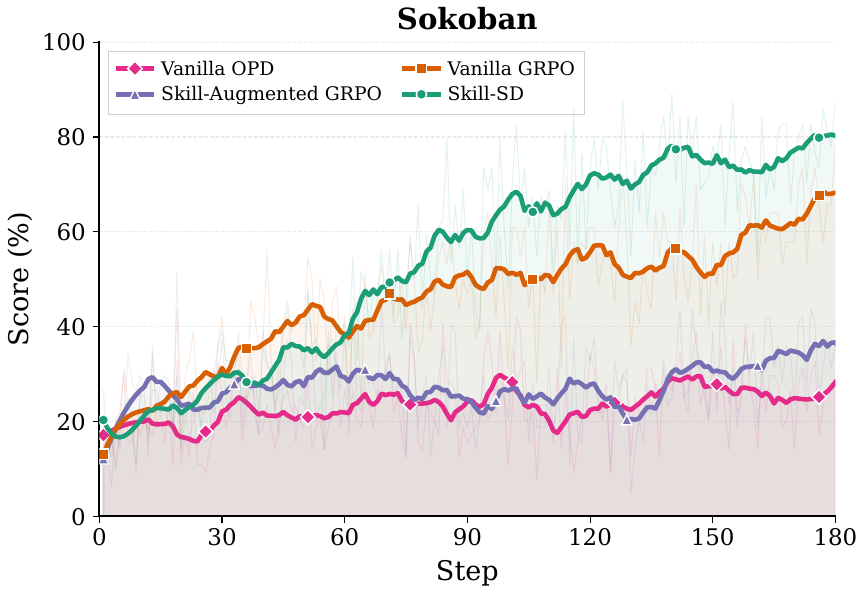}
\end{subfigure}
\caption{Training curves for all baselines on AppWorld (left) and Sokoban (right). On AppWorld, Skill-Augmented GRPO achieves the highest \emph{training} accuracy while \method steadily narrows the gap. On Sokoban, \method overtakes Vanilla GRPO after step~60 and maintains steady gains.}
\label{fig:training_baselines}
\end{figure}

\subsection{Ablation Analysis}
\label{sec:ablation}

We validate \method's two core design choices by examining what happens when each is replaced. The first choice is \emph{student-owned rollout}: in \method, the student generates trajectories under the plain prompt and the teacher re-scores them; the alternative is teacher-owned rollout, where the teacher generates trajectories under the skill-augmented prompt and the student re-forwards them. The second choice is \emph{dynamic teacher synchronization}: in \method, the teacher is synchronized from the latest student checkpoint at each iteration; the alternative is a frozen teacher that retains its initial parameters throughout training. Table~\ref{tab:ablation} reports the effect of each replacement.

\begin{table}[t]
\centering
\caption{Ablation: effect of replacing \method's design choices. \method uses on-policy student rollout with dynamic teacher synchronization (bolded). Subscripts denote absolute change from \method. $^\ast$Training collapsed during mid-training; values reflect the checkpoint before collapse.}
\label{tab:ablation}
\small
\setlength{\tabcolsep}{4pt}
\begin{tabular}{llcccccc}
\toprule
& & \multicolumn{2}{c}{\textbf{AppWorld}} & \multicolumn{2}{c}{\textbf{Sokoban}} & \multicolumn{2}{c}{\textbf{Avg.}} \\
\cmidrule(lr){3-4} \cmidrule(lr){5-6} \cmidrule(lr){7-8}
\textbf{Rollout} & \textbf{Teacher} & \textbf{Acc.} & \textbf{Comp.} & \textbf{Acc.} & \textbf{Comp.} & \textbf{Acc.} & \textbf{Comp.} \\
\midrule
On-policy & Frozen & 49.1\dDn{15.8} & 79.0\dDn{5.9} & 50.0\dDn{12.5} & 63.3\dDn{7.8} & 49.6\dDn{14.1} & 71.1\dDn{6.9} \\
\rowcolor{bestrowblue}
\textbf{On-policy} & \textbf{Dynamic} & \textbf{64.9}\dZr & \textbf{84.9}\dZr & \textbf{62.5}\dZr & \textbf{71.1}\dZr & \textbf{63.7}\dZr & \textbf{78.0}\dZr \\
\midrule
Off-policy$^\ast$ & Frozen & \color{gray}45.6\dDn{19.3} & \color{gray}78.8\dDn{6.1} & \color{gray}12.5\dDn{50.0} & \color{gray}31.3\dDn{39.8} & \color{gray}29.1\dDn{34.6} & \color{gray}55.0\dDn{23.0} \\
Off-policy$^\ast$ & Dynamic & \color{gray}42.1\dDn{22.8} & \color{gray}76.5\dDn{8.4} & \color{gray}10.9\dDn{51.6} & \color{gray}32.0\dDn{39.1} & \color{gray}26.5\dDn{37.2} & \color{gray}54.3\dDn{23.7} \\
\bottomrule
\end{tabular}
\end{table}

\begin{figure}[t]
\centering
\begin{subfigure}[t]{0.49\textwidth}
    \centering
    \includegraphics[width=\textwidth]{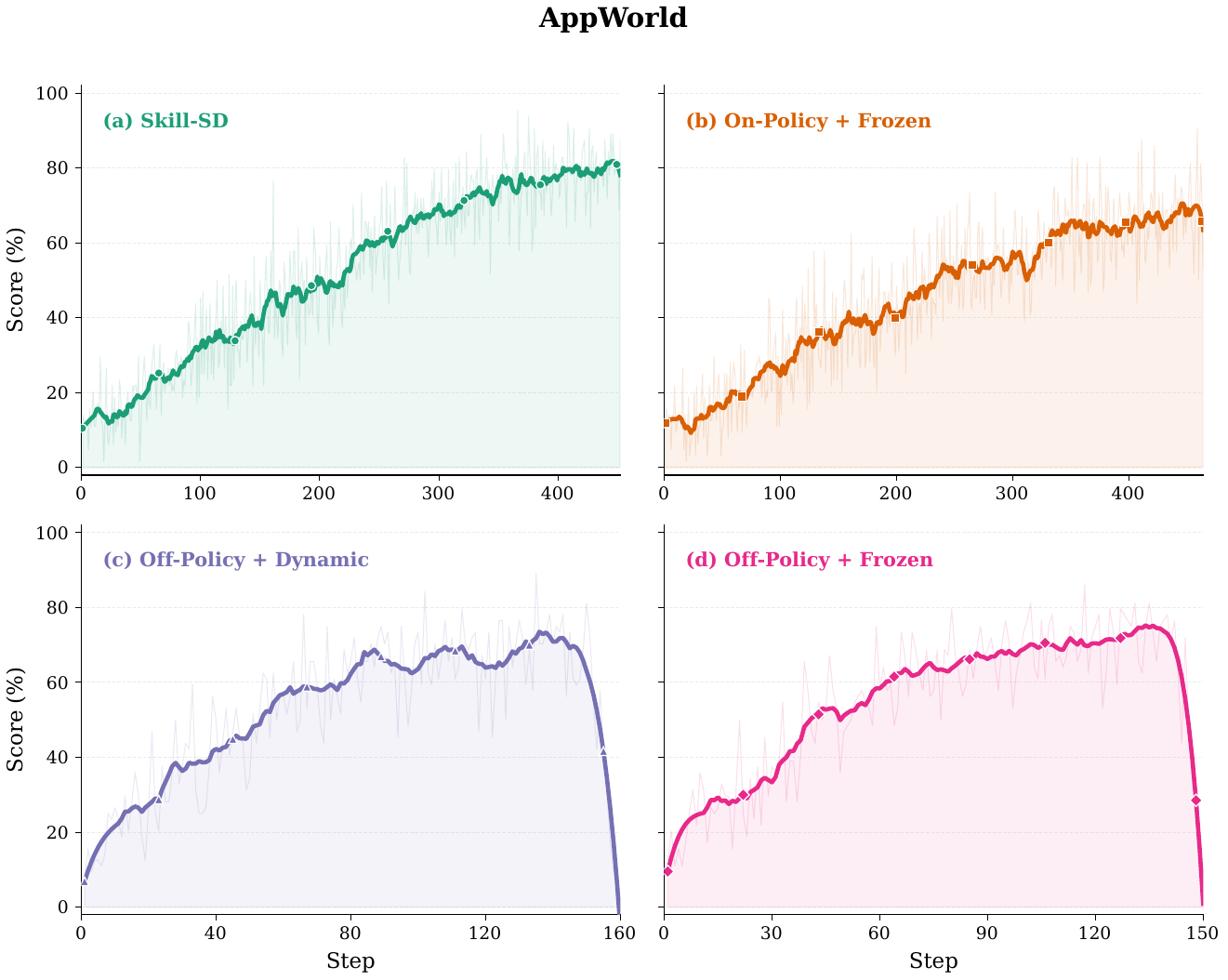}
\end{subfigure}
\hfill
\begin{subfigure}[t]{0.49\textwidth}
    \centering
    \includegraphics[width=\textwidth]{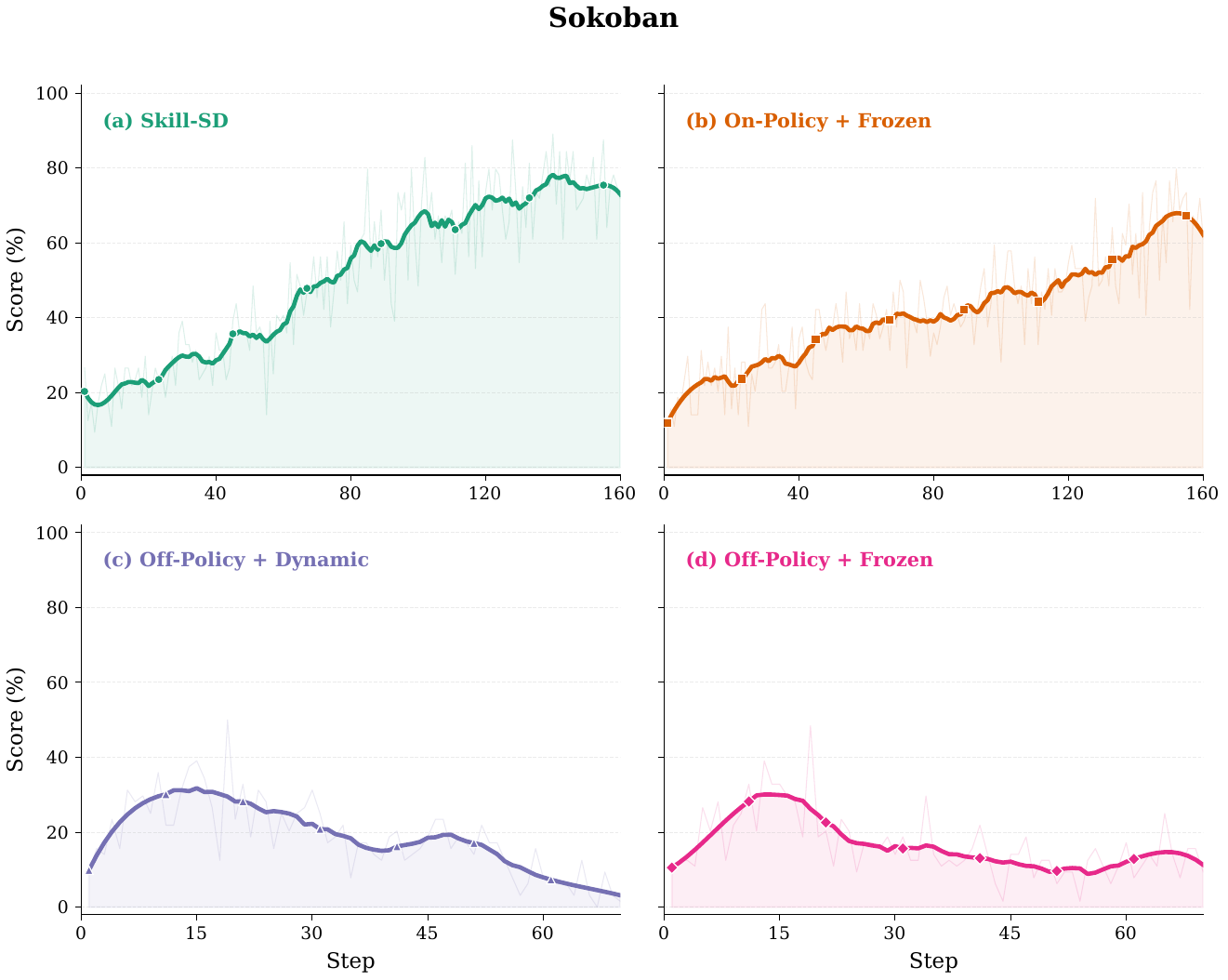}
\end{subfigure}
\caption{Training dynamics of the four teacher--student configurations on AppWorld (left) and Sokoban (right). \method (a, on-policy + dynamic) trains stably to high performance. On-policy + frozen (b) converges stably but to a lower plateau. Both off-policy variants (c,~d) achieve strong early performance but collapse during mid-training. The pattern replicates across both environments.}
\label{fig:ablation_curves}
\end{figure}

\paragraph{Student-owned rollout is essential.}
Teacher-owned (off-policy) rollout achieves strong early performance but collapses during mid-training in both environments (Figure~\ref{fig:ablation_curves}c,d), regardless of whether the teacher is frozen or dynamically refreshed. As the student policy diverges from the teacher's rollout distribution, the importance ratios become increasingly unstable, eventually destabilizing training. The collapse is far more severe on Sokoban (12.5\% accuracy, matching the base model without agent finetuning) than on AppWorld (45.6\%). This asymmetry reflects the environments' different error tolerances: in AppWorld, a suboptimal action from distribution mismatch can often be corrected in subsequent API calls, so performance degrades gradually; in Sokoban, a single wrong push can render the puzzle unsolvable, so even mild distribution drift causes catastrophic failure. This confirms that sustainable self-distillation requires student-owned rollout, and that the cost of off-policy mismatch scales with task irreversibility.

\paragraph{Dynamic synchronization improves on-policy training.}
Within the on-policy regime, dynamic teacher synchronization provides consistent gains over a frozen teacher: +15.8\% accuracy on AppWorld and +12.5\% on Sokoban. The frozen on-policy configuration remains stable but converges to a lower performance plateau.

\paragraph{Skills should guide the teacher, not the student.}
Skill-Augmented GRPO underperforms Vanilla GRPO on both benchmarks (42.1\% vs.\ 50.9\% on AppWorld; 20.3\% vs.\ 51.6\% on Sokoban), showing that directly injecting skills into the student prompt is counterproductive. This is analogous to the privileged-information framework of~\citet{vapnik2009lupi}: Skill-Augmented GRPO optimizes return under the augmented policy $\pi^E_\theta(a|h) := \pi_\theta(a|h, e)$ conditioned on skills, but at evaluation the agent uses the restricted policy $\pi^0_\theta(a|h) := \pi_\theta(a|h, \varnothing)$ without skills.

Even with infinite data, maximizing $J_E(\theta) = \mathbb{E}_{\tau \sim \pi^E_\theta}[R(\tau)]$ does not in general maximize $J_0(\theta) = \mathbb{E}_{\tau \sim \pi^0_\theta}[R(\tau)]$: the two objectives correspond to different conditional policies that share parameters $\theta$. A gradient step $\theta^+ = \theta + \alpha\,g_E$ improves the evaluation objective only to the extent that $g_0^\top g_E > 0$, and the alignment decays when skill-conditioned tokens dominate the computation or compete with the representations needed for planning.

Inspecting the training dynamics (Figure~\ref{fig:training_baselines}) corroborates this: on AppWorld, Skill-Augmented GRPO achieves the highest training score of any method yet collapses on validation, providing direct evidence of overfitting to the skill-conditioned policy; on Sokoban, skills actively disrupt spatial reasoning even during training. \method avoids this pitfall by confining skills to a training-only teacher branch, analogous to System~2$\to$System~1 compilation~\mbox{\citep{yu2024distilling}}, so that the student internalizes useful components through gradient-based distillation rather than prompt-level injection.

\paragraph{Pure self-distillation without reward signals fails.}
Vanilla OPD achieves only 22.8\% accuracy on AppWorld and 21.9\% on Sokoban, far below every reward-optimized variant. Privileged-context self-distillation alone cannot replace reward-driven learning in long-horizon interactive tasks.

\begin{figure}[t]
\centering
\includegraphics[width=0.95\textwidth]{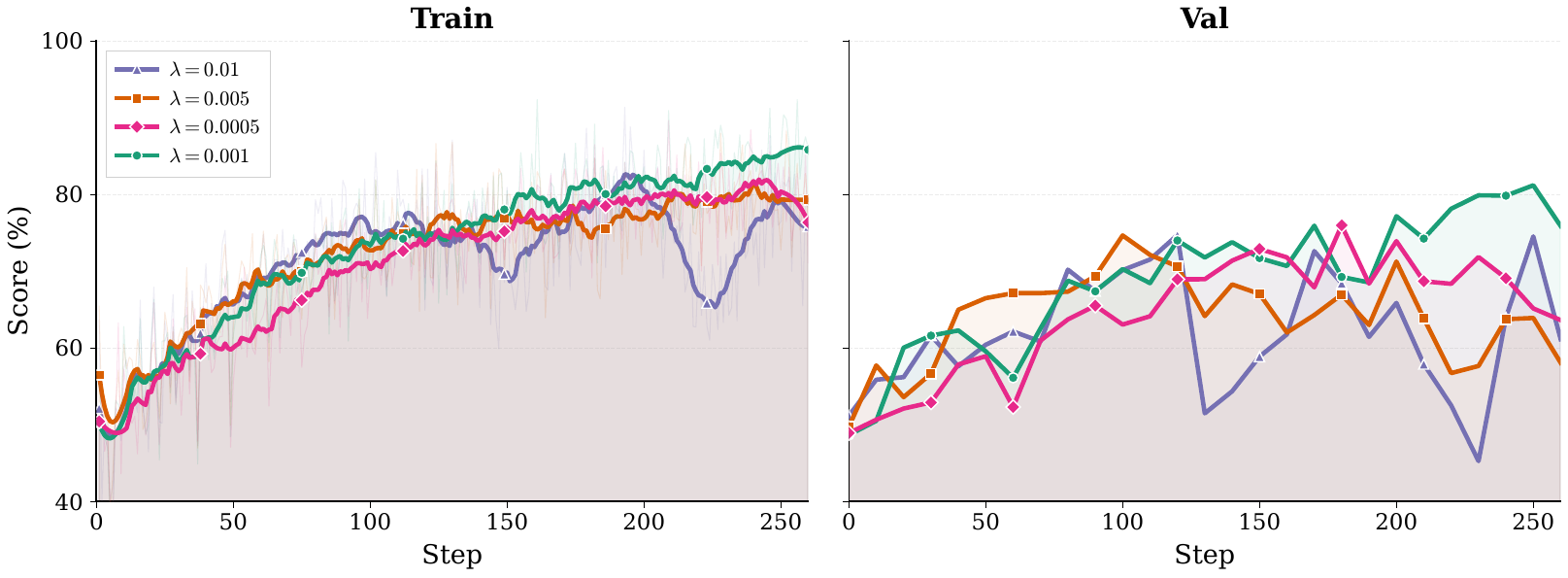}
\caption{Effect of SDL coefficient $\lambda$ on AppWorld training (left) and validation (right) completion rate. $\lambda = 0.001$ achieves the best validation performance; $\lambda = 0.01$ over-regularizes and suppresses exploration, while $\lambda = 0.0005$ provides insufficient teacher guidance.}
\label{fig:lambda_ablation}
\end{figure}

\paragraph{SDL weight $\lambda$ controls the RL--distillation balance.}
Figure~\ref{fig:lambda_ablation} sweeps $\lambda \in \{0.01, 0.005, 0.001, 0.0005\}$ on AppWorld. All four values produce similar training curves in the first 100 steps; after that, $\lambda = 0.001$ continues to improve and reaches the highest validation completion rate ($\sim$80\% with extended training beyond the checkpoint used in Table~\ref{tab:appworld_main}), while $\lambda = 0.01$ plateaus and becomes unstable. Reducing $\lambda$ below the optimum also hurts: $\lambda = 0.0005$ provides insufficient teacher guidance and performs comparably to the larger but noisier $\lambda = 0.005$, both settling around 70--75\% validation completion rate. The pattern reflects a bias--variance trade-off: larger $\lambda$ regularizes toward the teacher but overpowers the RL signal; smaller $\lambda$ recovers nearly pure GRPO and loses the dense supervision benefit. At $\lambda = 0.001$, SDL acts as a mild shaping term that guides the student without dominating the combined gradient.

\subsection{Training Efficiency}
\label{sec:efficiency}

The primary computational overhead of \method relative to vanilla GRPO is the external LLM API calls used to summarize trajectories into skills. In a na\"ive synchronous implementation, these API requests would block the rollout and parameter-update pipeline, substantially increasing wall-clock time per training step. We adopt a fully asynchronous architecture that decouples skill generation from the main training loop: API requests are dispatched in the background and their results are consumed by subsequent iterations rather than the current one. With this design, the API latency is largely overlapped with rollout and gradient computation, so the additional wall-clock overhead per training step is marginal.

\section{Conclusion}
\label{sec:conclusion}

We introduced \method, a framework that combines skill-conditioned self-distillation with GRPO for training multi-turn LLM agents. The teacher is conditioned on analytical skills distilled from trajectories; the student generates on-policy trajectories under the plain task prompt. This separation gives \method per-token supervision without constraining exploration. Dynamic teacher synchronization keeps the distillation signal calibrated as the student improves, and the importance-weighted SDL loss ensures per-token unbiased gradients for the distillation loss (Proposition~\ref{prop:sdl_grad}). These gains are achieved with only 90 AppWorld tasks and 96 Sokoban levels, showing that \method remains effective under the small task pools typical of multi-turn interactive benchmarks. On both AppWorld, an API-based multi-app coordination benchmark, and Sokoban, a spatial planning benchmark with irreversible moves, \method substantially outperforms vanilla GRPO, skill-augmented GRPO, and pure on-policy distillation. The lesson: skills should guide the teacher, not the student. Off-policy variants suffer mid-training collapse despite strong early performance, and frozen teachers converge to lower plateaus, indicating that the co-evolution of teacher and student, rather than any single architectural choice, drives the gains.

\section*{Limitations}
\label{sec:limitations}

We discuss several directions for improvement. First, skill retrieval currently uses a lightweight UCB bandit criterion rather than embedding-based semantic retrieval; while this keeps the training pipeline simple and avoids introducing an additional model, a learned retrieval component may improve skill selection as the bank grows. Second, we use sampled-token distillation (evaluating only the generated token) rather than full-vocabulary distillation; although this is far cheaper for long-horizon traces, partial-vocabulary approximations may improve the fidelity of the self-evolution signal. Third, our gradient analysis is token-level and per-update; extending it to trajectory-level KL bounds under evolving teacher references remains an open theoretical question.

\bibliography{references}
\bibliographystyle{colm2026_conference}

\appendix

\section{Skill Format}
\label{app:skill}

Each skill is a structured JSON object with three fields summarizing a completed trajectory. During training, skills are prepended to the teacher prompt to provide task-specific guidance.

\begin{examplebox}[Example Skill]
\textbf{success\_analysis:} ``The strategic intent to use specialized apps (Splitwise for identifying roommates, Venmo for sending payment requests) and prioritize authentication for restricted APIs was sound, as these tools are designed to streamline expense splitting and peer-to-peer payments.''

\textbf{mistake\_analysis:} ``Key failures stemmed from unvalidated assumptions (using a hypothetical bill amount instead of retrieving the actual receipt from the file system) and neglecting to verify API parameter requirements upfront, leading to repeated authentication errors.''

\textbf{golden\_workflow:} ``1.~Access the file system to retrieve the actual internet bill amount from the stored receipt. 2.~Authenticate with a contact app using valid credentials to obtain roommate emails. 3.~Authenticate with Venmo, calculate each roommate's equal share, then send payment requests via the correct API.''
\end{examplebox}

\section{Trajectory Example and Skill Generation}
\label{app:trajectory}

We present a complete AppWorld interaction trajectory and the skill generated from it. The task requires exporting all Spotify library data to a CSV file, then terminating the account. This trajectory was collected during training and is shown with minimal editing (long API outputs are truncated with \texttt{[...]}).

\smallskip
\noindent\textbf{Task:} \emph{Export a unique list of all the songs in my song and album library and all playlists in my Spotify account into ``\textasciitilde/backups/spotify.csv''. The file should have headers ``Title'' and ``Artists'' (artists separated by ``|''). Terminate my account after this backup is complete.}
\smallskip

\begin{agentbox}[Agent -- Step 1]
To solve this task, we need to: (1)~Access the Spotify account using login. (2)~Retrieve the full list of songs, albums, and playlists. (3)~Extract song titles and artists. (4)~Combine into a unique list. (5)~Export to CSV. (6)~Terminate the account.

\texttt{passwords = apis.supervisor.show\_account\_passwords()}
\end{agentbox}

\begin{envbox}[Environment -- Step 1]
\texttt{[\{"account\_name": "spotify", "password": "CgR\}U-F"\}, \{"account\_name": "file\_system", "password": "+ok\}\#yD"\}, ...]}
\end{envbox}

\begin{agentbox}[Agent -- Step 2]
The Spotify password is \texttt{"CgR\}U-F"}. I will now log in to Spotify.

\texttt{login\_result = apis.spotify.login(username='ka\_burt@gmail.com', password=spotify\_password)}
\end{agentbox}

\begin{envbox}[Environment -- Step 2]
\texttt{\{"access\_token": "eyJhb...", "token\_type": "Bearer"\}}
\end{envbox}

\begin{agentbox}[Agent -- Step 3]
Login was successful. Now let me check the available Spotify APIs.

\texttt{print(apis.api\_docs.show\_api\_descriptions(app\_name='spotify'))}
\end{agentbox}

\begin{envbox}[Environment -- Step 3]
\texttt{[show\_song\_library, show\_album\_library, show\_playlist\_library, show\_song, delete\_account, ...]} \emph{[74 API endpoints; truncated]}
\end{envbox}

\begin{agentbox}[Agent -- Step 4]
The relevant APIs are \texttt{show\_song\_library}, \texttt{show\_album\_library}, \texttt{show\_playlist\_library}, and \texttt{show\_song}. Let me retrieve the song library first.

\texttt{song\_library = apis.spotify.show\_song\_library(access\_token=spotify\_access\_token)}
\end{agentbox}

\begin{envbox}[Environment -- Step 4]
Returns 5 songs with titles and artists. E.g., \texttt{\{"song\_id": 202, "title": "Summer's End", "artists": [\{"name": "Ethan Wallace"\}]\}}, ...
\end{envbox}

\begin{agentbox}[Agent -- Step 5]
Now let me retrieve the album library and playlist library.

\texttt{album\_library = apis.spotify.show\_album\_library(access\_token=...)}

\texttt{playlists = apis.spotify.show\_playlist\_library(access\_token=...)}
\end{agentbox}

\begin{envbox}[Environment -- Step 5]
Returns 5 albums (16 song IDs total) and 5 playlists (39 song IDs total).
\end{envbox}

\begin{agentbox}[Agent -- Step 6]
Now I need to fetch details for each song ID from albums and playlists using \texttt{show\_song}.

\texttt{song = apis.spotify.show\_song(access\_token=spotify\_access\_token, song\_id=song\_id)}
\end{agentbox}

\begin{envbox}[Environment -- Step 6: {\color{errorred}\textbf{Error}}]
\texttt{Exception: Unexpected parameter 'access\_token' passed to the show\_song API. Allowed parameters are: ['song\_id']}
\end{envbox}

\begin{agentbox}[Agent -- Step 7]
I made a mistake in the API call. Let me check the API specification for \texttt{show\_song}.

\texttt{print(apis.api\_docs.show\_api\_doc(app\_name='spotify', api\_name='show\_song'))}
\end{agentbox}

\begin{envbox}[Environment -- Step 7]
\texttt{\{"parameters": [\{"name": "song\_id", "type": "integer", "required": true\}]\}} --- confirms \texttt{show\_song} only takes \texttt{song\_id}, no \texttt{access\_token}.
\end{envbox}

\begin{agentbox}[Agent -- Step 8]
Confirmed. I will fix the code to call \texttt{show\_song(song\_id=...)} without \texttt{access\_token}, and collect all songs from song library, albums, and playlists into a deduplicated list.

\texttt{song = apis.spotify.show\_song(song\_id=song\_id)} \emph{[loops over all song IDs]}
\end{agentbox}

\begin{envbox}[Environment -- Step 8]
\texttt{59} --- 59 unique songs extracted successfully.
\end{envbox}

\begin{agentbox}[Agent -- Step 9]
59 unique songs collected. Now I need to write the CSV. Let me check the \texttt{file\_system} APIs and log in.

\texttt{file\_system\_login = apis.file\_system.login(password=file\_system\_password)}
\end{agentbox}

\begin{envbox}[Environment -- Step 9: {\color{errorred}\textbf{Error}}]
\texttt{Validation error. Reason: username: Field required}
\end{envbox}

\begin{agentbox}[Agent -- Step 10]
The \texttt{file\_system.login} requires both \texttt{username} and \texttt{password}. Let me check the API docs.

\texttt{print(apis.api\_docs.show\_api\_doc(app\_name='file\_system', api\_name='login'))}
\end{agentbox}

\begin{envbox}[Environment -- Step 10]
Confirms: \texttt{username} (string, required) and \texttt{password} (string, required).
\end{envbox}

\begin{agentbox}[Agent -- Step 11]
Now I will log in with both username and password, create the CSV, and write it.

\texttt{apis.file\_system.login(username='ka\_burt@gmail.com', password=...)}

\texttt{apis.file\_system.create\_file(file\_path='\textasciitilde/backups/spotify.csv', content=csv\_content, ...)}
\end{agentbox}

\begin{envbox}[Environment -- Step 11]
\texttt{CSV file created successfully.}
\end{envbox}

\begin{agentbox}[Agent -- Step 12]
Backup complete. Now terminate the Spotify account.

\texttt{apis.spotify.delete\_account(access\_token=spotify\_access\_token)}
\end{agentbox}

\begin{envbox}[Environment -- Step 12]
\texttt{Spotify account terminated successfully.}
\end{envbox}

\begin{agentbox}[Agent -- Step 13]
All tasks completed. Marking as done.

\texttt{apis.supervisor.complete\_task()}
\end{agentbox}

\begin{envbox}[Environment -- Step 13]
\texttt{Execution successful.}
\end{envbox}

\paragraph{Generated skill.} The following skill is produced from this trajectory by the auxiliary summarizer:

\begin{examplebox}[Generated Skill]
\textbf{success\_analysis:} ``The strategic approach of systematically retrieving data from song library, album library, and playlists before combining into a deduplicated export was sound, as was the correct sequencing of backup-then-terminate.''

\textbf{mistake\_analysis:} ``Two API parameter errors occurred: (1) passing \texttt{access\_token} to \texttt{show\_song}, which only accepts \texttt{song\_id}; (2) omitting \texttt{username} when calling \texttt{file\_system.login}. Both stemmed from not consulting API specifications before the first call.''

\textbf{golden\_workflow:} ``1.~Retrieve credentials via supervisor. 2.~Log in to Spotify and collect song/album/playlist libraries. 3.~For each song ID, call \texttt{show\_song(song\_id=...)} without extra parameters. 4.~Log in to file system with both username and password, then write deduplicated CSV. 5.~Delete Spotify account and mark task complete.''
\end{examplebox}

\section{Prompt Templates}
\label{app:prompts}

The AppWorld system prompt is adapted from the official AppWorld benchmark~\citep{trivedi2024appworld}. The full prompt including all API specifications is available in the benchmark repository.

\begin{promptbox}[AppWorld System Prompt (abbreviated)]
I am your supervisor and you are a super intelligent AI Assistant whose job is to achieve my day-to-day tasks completely autonomously.

To do this, you will need to interact with app/s (e.g., spotify, venmo etc) using their associated APIs on my behalf. For this you will undertake a \emph{multi-step conversation} using a python REPL environment.

Here are three key APIs that you need to know to get more information

\texttt{\# To get a list of apps that are available to you.}\\
\texttt{print(apis.api\_docs.show\_app\_descriptions())}

\texttt{\# To get the list of apis under any app, e.g.\ spotify}\\
\texttt{print(apis.api\_docs.show\_api\_descriptions(app\_name='spotify'))}

\texttt{\# To get the specification of a particular api}\\
\texttt{print(apis.api\_docs.show\_api\_doc(app\_name='spotify', api\_name='login'))}

Each code execution will produce an output that you can use in subsequent calls. Using these APIs, you can now generate code, that I will execute, to solve the task. Let's start with the task

My name is: \ph{main\_user.first\_name} \ph{main\_user.last\_name}. My personal email is \ph{main\_user.email} and phone number is \ph{main\_user.phone\_number}.\\
Task: How many playlists do I have in Spotify?

[\emph{A multi-turn few-shot demonstration follows, showing authentication, API discovery, pagination, and calling} \texttt{apis.supervisor.complete\_task()}\emph{. See the AppWorld benchmark repository for the full prompt.}]

\textbf{Key instructions:}\\
(1) Make sure to end code blocks with \texttt{\textasciigrave\textasciigrave\textasciigrave} followed by a newline.\\
(2) Remember you can use the variables in your code in subsequent code blocks.\\
(3) Always look at API specifications before calling an API.\\
(4) Write small chunks of code and only one chunk per step.\\
(5) Many APIs return items in ``pages''. Loop over \texttt{page\_index}.\\
(6) Once you have completed the task, call \texttt{apis.supervisor.complete\_task()}.

My name is: \ph{first\_name} \ph{last\_name}. My personal email is \ph{email} and phone number is \ph{phone}.\\
Task: \ph{input\_str}
\end{promptbox}

\begin{promptbox}[Sokoban System Prompt]
You are an expert agent operating in the Sokoban environment.

\textbf{\# Symbols and Their Meaning}\\
- Walls (\texttt{\#}): These block movement. You can't move through or push anything into walls.\\
- Floor (\texttt{\_}): Open spaces where you can walk and move boxes.\\
- Targets (\texttt{O}): The spots where boxes need to go.\\
- Boxes (\texttt{X}): These are what you need to push onto the targets.\\
- Player (\texttt{P}): That's you! You'll move around the grid to push boxes.\\
- Box on Target (\texttt{T}): A box successfully placed on a target.\\
- Player on Target (\texttt{S}): You standing on a target.

\textbf{\# Your Goal}\\
Your goal is to push all the boxes (\texttt{X}) onto the target spots (\texttt{O}). Once all boxes are on the targets, you win!

\textbf{\# Rules}\\
1. You can only move up, down, left, or right. And don't overthink it.\\
2. You can only push boxes. You can't pull them, so plan ahead to avoid getting stuck.\\
3. You can't walk through or push boxes into walls (\texttt{\#}).\\
4. To avoid traps, do not push boxes into corners or against walls where they can't be moved again.

You will be provided the current map observation.\\
After your thought process, you MUST output your final single action enclosed in \texttt{\textasciigrave\textasciigrave\textasciigrave{} \textasciigrave\textasciigrave\textasciigrave}.\\
Valid Actions: \texttt{up} $|$ \texttt{down} $|$ \texttt{left} $|$ \texttt{right}
\end{promptbox}

\begin{promptbox}[Skill Generation Prompt]
\textbf{\# Trajectory Analysis and Golden Skill Generation}

You are an expert trajectory analyst. Your task is to analyze the given trajectory and generate a structured golden skill that captures key learnings.

\textbf{\#\# Conversation History}\\
\ph{conversation\_history}

\textbf{\#\# Trajectory Outcome}\\
- \textbf{Result}: \ph{result\_status}\\
- \textbf{Total Steps}: \ph{step\_count}\\
- \textbf{Completion Rate}: \ph{completion\_rate}

\textbf{\#\# Output Requirements (JSON)}

You need to deeply analyze the Trajectory Outcome and Conversation History I provided, thinking step by step, and then summarize the skill to guide the model in future encounters.

Output a unique JSON object containing exactly these 3 fields:

1. \texttt{success\_analysis}: A comprehensive paragraph summarizing the \textbf{strategic intent} that led (or would lead) to success.\\
2. \texttt{mistake\_analysis}: A comprehensive paragraph analyzing the \textbf{root causes} of failures or inefficiencies observed in the trajectory.\\
3. \texttt{golden\_workflow}: A concise description of the \textbf{ideal overall process} to complete this task from start to finish.

\textbf{\#\# REMEMBER}

The analysis must be concise. As an expert, the skill you summarize will guide the model's subsequent operations. Therefore, it's crucial to adopt a guiding approach, focusing primarily on the steps that confuse the model.\\
Your success-to-failure output ratio should follow the Completion Rate (\ph{completion\_rate}). If the Completion Rate is high, you should focus on summarizing the success analysis; otherwise, you should focus on summarizing the mistake analysis.
\end{promptbox}

\section{Importance-Weighted SDL under Student and Teacher Rollouts}
\label{app:sdl_gradient}

This appendix formalizes the sampled-token self-distillation term used in Eq.~\eqref{eq:sdl}. The analysis follows the modern distinction between \emph{KL value estimation} and \emph{KL gradient optimization}: a bare \(k_3\) estimator can be a good value estimator while still producing the wrong gradient when differentiated directly~\citep{schulman2020kl,tang2025pitfalls,liu2025rethinking}. That distinction matters here because the auxiliary loss is interpreted as self-evolution. If its gradient is biased under the actual sampling policy, then the loss may still move the model, but it no longer cleanly implements the intended teacher-to-student transfer direction.

\subsection{Estimator family}
\label{app:kl_family}

Let \(q_\theta\) denote the current student distribution, \(p\) a fixed teacher reference distribution, and \(\mu\) the actual sampling distribution at a fixed token position. Define
\begin{equation}
\ell = \log \frac{q_\theta(y)}{p(y)}.
\end{equation}
The standard reverse-KL estimator family includes
\begin{align}
k_1(\ell) &= \ell, \\
k_2(\ell) &= \tfrac{1}{2}\ell^2, \\
k_3(\ell) &= e^{-\ell} - 1 + \ell.
\end{align}
Among these, \(k_3\) is attractive because it is non-negative and is widely used as a low-variance reverse-KL estimator~\citep{schulman2020kl}. But recent analyses show that differentiating \(k_3\) directly does not in general yield the desired reverse-KL gradient~\citep{tang2025pitfalls,liu2025rethinking}.

\subsection{General sampled-token SDL}
\label{app:sdl_quantities}

\begin{definition}[Importance-weighted SDL]
\label{def:sdl}
For a fixed prefix, let \(y_t \sim \mu(\cdot \mid y_{<t})\). Define
\begin{align}
\ell_t &\triangleq \log q_\theta(y_t \mid y_{<t}) - \log p(y_t \mid y_{<t}), \label{eq:ell_app} \\
\rho_t^{\mu} &\triangleq \frac{q_\theta(y_t \mid y_{<t})}{\operatorname{sg}\!\left(\mu(y_t \mid y_{<t})\right)}, \label{eq:rho_app}
\end{align}
where \(\operatorname{sg}(\cdot)\) denotes stop-gradient.
\end{definition}

The choice of \(\rho_t^\mu\) is structurally forced. Its denominator must match the actual sampling distribution \(\mu\), otherwise the importance correction does not close under the sampling expectation. Its numerator must be the trainable student, otherwise the product rule would mix fixed-reference terms with the wrong score function and the resulting gradient would not reduce to the intended reverse-KL form.

\begin{proposition}[Per-token SDL gradient identity]
\label{prop:sdl_grad}
Under Definition~\ref{def:sdl},
\begin{equation}
\nabla_\theta\!\bigl(\rho_t^\mu \cdot k_3(\ell_t)\bigr)
=
\rho_t^\mu \cdot s_\theta(y_t) \cdot \ell_t,
\label{eq:sdl_grad}
\end{equation}
where \(s_\theta(y_t) = \nabla_\theta \log q_\theta(y_t \mid y_{<t})\). Moreover, for any fixed prefix \(y_{<t}\),
\begin{equation}
\expect_{y_t \sim \mu(\cdot \mid y_{<t})}
\!\left[\nabla_\theta\!\bigl(\rho_t^\mu \cdot k_3(\ell_t)\bigr)\right]
=
\nabla_\theta D_{\kl}\!\bigl(q_\theta(\cdot \mid y_{<t}) \,\|\, p(\cdot \mid y_{<t})\bigr).
\end{equation}
\end{proposition}

\begin{proof}
We apply the product rule:
\begin{equation}
\nabla_\theta(\rho_t^\mu \cdot k_3) = (\nabla_\theta \rho_t^\mu) \cdot k_3 + \rho_t^\mu \cdot (\nabla_\theta k_3).
\end{equation}
Because the denominator of \(\rho_t^\mu\) is frozen under stop-gradient,
\begin{equation}
\nabla_\theta \rho_t^\mu = \rho_t^\mu \cdot s_\theta(y_t).
\end{equation}
Since \(\log p(y_t \mid y_{<t})\) is computed at the pre-update parameters and does not depend on the current \(\theta\), \(\nabla_\theta \ell_t = s_\theta(y_t)\). Using \(k_3'(\ell) = 1 - e^{-\ell}\), we obtain
\begin{equation}
\nabla_\theta k_3(\ell_t) = (1 - e^{-\ell_t}) \cdot s_\theta(y_t).
\end{equation}
Substituting both expressions gives
\begin{align}
\nabla_\theta(\rho_t^\mu \cdot k_3)
&=
\rho_t^\mu \cdot s_\theta(y_t)\cdot k_3(\ell_t)
+
\rho_t^\mu \cdot (1 - e^{-\ell_t}) \cdot s_\theta(y_t) \\
&=
\rho_t^\mu \cdot s_\theta(y_t)
\left[k_3(\ell_t) + 1 - e^{-\ell_t}\right].
\end{align}
Using \(k_3(\ell_t) = e^{-\ell_t} - 1 + \ell_t\), the bracket simplifies to \(\ell_t\), proving Eq.~\eqref{eq:sdl_grad}. Taking expectation under the sampling distribution yields
\begin{align}
\expect_{y_t \sim \mu(\cdot \mid y_{<t})}
\!\left[\nabla_\theta(\rho_t^\mu \cdot k_3)\right]
&=
\sum_{y_t} \mu(y_t \mid y_{<t})
\frac{q_\theta(y_t \mid y_{<t})}{\mu(y_t \mid y_{<t})}
s_\theta(y_t)\ell_t \\
&=
\sum_{y_t} q_\theta(y_t \mid y_{<t}) s_\theta(y_t)\ell_t \\
&=
\nabla_\theta D_{\kl}\!\bigl(q_\theta(\cdot \mid y_{<t}) \,\|\, p(\cdot \mid y_{<t})\bigr),
\end{align}
which proves the conditional unbiasedness claim.
\end{proof}

\subsection{On-policy and off-policy corollaries}
\label{app:corollaries}

\paragraph{On-policy main method.}
For the main \method update, \(\mu = \pi_{\theta_{\text{old}}}^{\text{stu}}(\cdot \mid x,y_{<t})\) and \(p = \pi_{\bar{\theta}}^{\text{tea}}(\cdot \mid x,S(x),y_{<t})\). Therefore
\begin{equation}
\rho_t^{\text{on}}
=
\frac{\pi_\theta^{\text{stu}}(y_t \mid x,y_{<t})}
{\operatorname{sg}\!\left(\pi_{\theta_{\text{old}}}^{\text{stu}}(y_t \mid x,y_{<t})\right)}
\end{equation}
is the importance weight for the main self-evolution signal. Numerically, \(\rho_t^{\text{on}}\) equals the GRPO ratio \(r_t\), but the role is different: \(\rho_t^{\text{on}}\) weights the auxiliary reverse-KL term, whereas \(r_t\) defines the clipped reinforcement-learning surrogate.

\paragraph{Off-policy comparison branch.}
For the teacher-rollout comparison branch, \(\mu = \pi_{\theta_{\text{old}}}^{\text{tea}}(\cdot \mid x,S(x),y_{<t})\), which yields
\begin{equation}
\rho_t^{\text{off}}
=
\frac{\pi_\theta^{\text{stu}}(y_t \mid x,y_{<t})}
{\operatorname{sg}\!\left(\pi_{\theta_{\text{old}}}^{\text{tea}}(y_t \mid x,S(x),y_{<t})\right)}.
\end{equation}
This is the off-policy correction emphasized in RPG-style analyses~\citep{zhang2025rpg}: without it, the differentiated \(k_3\) term would not follow the intended reverse-KL transfer direction under teacher sampling.

\subsection{Why SDL and GRPO need distinct interpretations}
\label{app:dual_ratio}

\method uses two different quantities for two different jobs:
\begin{itemize}
    \item The main-branch SDL weight \(\rho_t^{\text{on}} = \pi_\theta^{\text{stu}} / \pi_{\theta_{\text{old}}}^{\text{stu}}\) must match the \emph{student sampling distribution}.
    \item The off-policy SDL weight \(\rho_t^{\text{off}} = \pi_\theta^{\text{stu}} / \pi_{\theta_{\text{old}}}^{\text{tea}}\) must match the \emph{teacher sampling distribution}.
    \item The GRPO clipping ratio \(r_t = \pi_\theta^{\text{stu}} / \pi_{\theta_{\text{old}}}^{\text{stu}}\) must stay centered at the \emph{old student policy}, even when we apply DAPO-style clip-higher asymmetric bounds~\citep{yu2025dapo}.
\end{itemize}
If one naively reused the teacher-denominator ratio for clipping, the trust region would no longer be centered at \(1\) because the teacher and student differ by prompt conditioning. This is exactly the type of mis-centering that recent off-policy analyses warn against~\citep{yan2025luffy,zhang2025rpg}. In the main on-policy branch, \(\rho_t^{\text{on}}\) happens to equal \(r_t\) numerically, but the interpretation is still different: the former defines the self-evolution update, whereas the latter defines the clipped RL surrogate.

\begin{table}[ht]
\centering
\small
\caption{Importance-weighted SDL and student-centered GRPO in \method.}
\label{tab:ratio_comparison}
\begin{tabular}{lll}
\toprule
\textbf{Quantity} & \textbf{Definition} & \textbf{Purpose} \\
\midrule
Main SDL weight \(\rho_t^{\text{on}}\) & \(\pi_\theta^{\text{stu}} / \operatorname{sg}(\pi_{\theta_{\text{old}}}^{\text{stu}})\) & Match student sampling in the main method \\
Off-policy SDL weight \(\rho_t^{\text{off}}\) & \(\pi_\theta^{\text{stu}} / \operatorname{sg}(\pi_{\theta_{\text{old}}}^{\text{tea}})\) & Match teacher sampling in the comparison branch \\
GRPO clipping ratio \(r_t\) & \(\pi_\theta^{\text{stu}} / \operatorname{sg}(\pi_{\theta_{\text{old}}}^{\text{stu}})\) & Center trust region at old student \\
\bottomrule
\end{tabular}
\end{table}

\section{Additional Experimental Details}
\label{app:exp_details}

\paragraph{Variant definition.}
Table~\ref{tab:variant_def} summarizes the exact axes varied in the current AppWorld study.

\begin{table}[ht]
\centering
\small
\caption{Variant definitions for the AppWorld ablation matrix.}
\label{tab:variant_def}
\begin{tabular}{llll}
\toprule
\textbf{Variant} & \textbf{Rollout Source} & \textbf{Teacher Refresh} & \textbf{Distillation} \\
\midrule
Off-policy + dynamic & Teacher & Dynamic self-evolution & Yes \\
Off-policy + frozen & Teacher & Frozen & Yes \\
On-policy + dynamic & Student & Dynamic self-evolution & Yes \\
On-policy + frozen & Student & Frozen & Yes \\
Skill-Augmented GRPO & Skill-conditioned student & N/A & No \\
Vanilla GRPO & Student & N/A & No \\
Vanilla OPD & Teacher/student re-forward pair & Frozen & Yes, no RL \\
Baseline & Qwen3-4B-Instruct-2507 & N/A & No \\
\bottomrule
\end{tabular}
\end{table}

\paragraph{AppWorld protocol.}
We use the official AppWorld train/dev split~\citep{trivedi2024appworld}. Accuracy is pass@1 on the dev set. Completion rate is the fraction of task-specific unit tests satisfied by the final environment state.

\paragraph{Loss reduction and clipping.}
Both \(\mathcal{L}_{\text{GRPO}}\) and \(\mathcal{L}_{\text{SDL}}\) use \texttt{verl}'s \texttt{token-mean} reduction: losses are summed over valid action tokens and divided by the total valid-token count. For GRPO, we use DAPO's clip-higher setting~\citep{yu2025dapo}, i.e., asymmetric clipping with \(\varepsilon_h > \varepsilon_l\).

\section{Training Hyperparameters}
\label{app:hyperparams}

We implement \method using the rLLM framework\footnote{\url{https://github.com/rllm-org/rllm}} with the verl backend\footnote{\url{https://github.com/volcengine/verl}}. Trajectory summarization into skills is performed by Seed1.8~\citep{seed2026seed18}. GRPO follows~\citet{shao2024deepseekmath}. Multi-turn rollouts use a \emph{Token-In-Token-Out} (TITO) mode, where the model reads and generates raw tokens rather than going through the chat completion API. This avoids the token-ID inconsistencies that arise when applying chat templates to multi-turn message histories---a discrepancy that can cause the training distribution to diverge from the rollout distribution and destabilize advantage estimation.

\begin{table}[h]
\centering
\small
\caption{Training hyperparameters for AppWorld and Sokoban experiments.}
\label{tab:hyperparams}
\begin{tabular}{lc}
\toprule
\textbf{Hyperparameter} & \textbf{Value} \\
\midrule
Group size $G$ (rollouts per task) & 4 \\
Train batch size & 16 \\
Optimization epochs per batch & 1 \\
Max prompt length & 8192 tokens \\
Max response length & 20480 tokens \\
Dynamic batch sizing & Enabled \\
Ulysses sequence parallel size & 4 \\
Max tokens per GPU & 14336 \\
Training temperature & 1.0 \\
Evaluation temperature & 0.0 (greedy) \\
GPU memory utilization & 0.7 \\
Entropy coefficient & 0.0 \\
DAPO clip $\varepsilon_l$ & 0.2 \\
DAPO clip $\varepsilon_h$ & 0.28 \\
Advantage normalization $\varepsilon$ & $1 \times 10^{-8}$ \\
Learning rate & $1 \times 10^{-6}$ \\
Weight decay & 0.0 \\
UCB exploration coefficient $c$ & $\sqrt{2}$ \\
SDL coefficient $\lambda$ & $0.001$ \\
Optimizer & AdamW \\
GPU & $8{\times}$NVIDIA L40S \\
\bottomrule
\end{tabular}
\end{table}

The UCB exploration coefficient $c = \sqrt{2}$ follows the UCB1 algorithm of~\citet{auer2002ucb1}, which is theoretically optimal for rewards bounded in $[0, 1]$ via Hoeffding's inequality, guaranteeing logarithmic regret.

\end{document}